\documentclass{article}

\usepackage{arxiv}

\usepackage{amsmath,amssymb,amsthm}
\usepackage{xspace}
\usepackage{epsfig}
\usepackage{makecell}
\usepackage{wasysym}
\usepackage{csquotes}
\usepackage{algorithm}
\usepackage[noend]{algorithmic}
\usepackage{color}
\usepackage{url}
\usepackage{booktabs,multirow}

\newtheorem{theorem}{Theorem}
\newtheorem{lemma}[theorem]{Lemma}
\newtheorem{corollary}[theorem]{Corollary}
\newcommand{\ignore}[1]{}
\renewcommand{\epsilon}{\varepsilon}

\newcommand{\E}[1]{\text{E}\left(#1\right)}

\newcommand{\Prob}[1]{\text{Pr}\left(#1\right)}

\newcommand{\EA}{\text{(1+1)~EA}\xspace}

\newcommand{\lRLS}{(1+$\lambda$)~RLS\xspace}

\newcommand{\diam}{\mathrm{diam}}

\newcommand{\chromatic}{\chi}

\definecolor{black}{rgb}{0.0,0.0,0.0}
\definecolor{grey}{rgb}{0.3,0.3,0.3}

\usepackage{tikz}
\usepackage{verbatim}
\usetikzlibrary{arrows,shapes,fit}

\tikzstyle{color1}=[fill=red!40!white]
\tikzstyle{color2}=[fill=blue!40!white]
\tikzstyle{color3}=[fill=green!50!white]
\tikzstyle{color4}=[fill=orange!50!white]
\tikzstyle{color5}=[fill=violet!30!white]
\tikzstyle{color6}=[fill=yellow!30!white]
\tikzstyle{vertex}=[anchor=center,draw,circle,inner sep=0cm,minimum size=.4cm, fill=white]

\newcommand{\thmref}[1]{{\footnotesize{}\color{gray!60!black}[Thm~\ref{#1}]}}

\newcommand{\aware}{tailored\xspace}
\newcommand{\Aware}{Tailored\xspace}
\newcommand{\unaware}{generic\xspace}

\title{More Effective Randomized Search Heuristics for Graph Coloring Through Dynamic Optimization}
\date{ }

\author{
  Jakob Bossek \\
  Optimisation and Logistics\\
  The University of Adelaide\\
  Adelaide, Australia
  \And 
  Frank Neumann \\
  Optimisation and Logistics\\
  The University of Adelaide\\
  Adelaide, Australia
  \And
  Pan Peng \\
  Dept. of Computer Science \\
  University of Sheffield \\
  Sheffield, United Kingdom
  \And
  Dirk Sudholt \\
  Dept. of Computer Science \\
  University of Sheffield \\
  Sheffield, United Kingdom
}

\begin{document}
\maketitle

\begin{abstract}
Dynamic optimization problems have gained significant attention in evolutionary computation as evolutionary algorithms (EAs) can easily adapt to changing environments. We show that EAs can solve the graph coloring problem for bipartite graphs more efficiently by using dynamic optimization. In our approach the graph instance is given incrementally such that the EA can reoptimize its coloring when a new edge introduces a conflict. We show that, when edges are inserted in a way that preserves graph connectivity, Randomized Local Search (RLS) efficiently finds a proper 2-coloring for all bipartite graphs. This includes graphs for which RLS and other EAs need exponential expected time in a static optimization scenario.
We investigate different ways of building up the graph by popular graph traversals such as breadth-first-search and depth-first-search and analyse the resulting runtime behavior.
We further show that offspring populations (e.\,g.\ a (1+$\lambda$)~RLS) lead to an exponential speedup in~$\lambda$. Finally, an island model using 3 islands succeeds in an optimal time of $\Theta(m)$ on every $m$-edge bipartite graph, outperforming offspring populations. This is the first example where an island model guarantees a speedup that is not bounded in the number of islands. 
\end{abstract}

% keywords can be removed

\section{Introduction}
Evolutionary computing techniques have been applied to a wide range of problems that involve stochastic and/or dynamic environments~\cite{DBLP:series/isrl/RichterY13}. These methods can easily adapt to new environments which makes them well suited to deal with dynamic changes~\cite{branke2012evolutionary,nguyen2012evolutionary}. Understanding the principle of reoptimization carried out by an evolutionary algorithm for a dynamically changing problem is an important task and we contribute to this area by studying dynamic variants of the well-known graph coloring problem. Our main message is that a static combinatorial optimization problem may be solved more efficiently in a dynamic setup than in a static one.

Studies around dynamic optimization in the context of evolutionary algorithms have focused on the type, magnitude and frequency of changes that occur in the problem that is changing dynamically over time. Different types of experimental and theoretical studies have been carried out. Those experimental studies usually consider a benchmark that may be obtained from a classical static problem by applying specific dynamic changes to the static problem formulation over time~\cite{DBLP:conf/ppsn/Roostapour0N18,DBLP:journals/corr/abs-1811-07806}. 
A wide range of studies on the runtime behavior of evolutionary computing techniques for dynamic and stochastic problems have been carried out in recent years. We refer the reader to \cite{BookDoeNeu} for an overview. These studies build on a larger body of mathematical methods for the analysis of evolutionary computing techniques developed over the last 20 years (see~\cite{DBLP:books/daglib/0025643,Auger11,ncs/Jansen13,BookDoeNeu}~for comprehensive presentations).
Theoretical investigations in terms of runtime analysis for dynamic problems usually focus on the reoptimization time which measures the amount of time that an algorithm needs to recompute an optimal solution when a dynamic change has happened to a static problem for which an optimal solution has been obtained. Other studies for $\mathcal{NP}$-hard problems also consider the task of recomputing a good approximation after a dynamic change has occurred. Such studies include makespan scheduling~\cite{DBLP:conf/ijcai/NeumannW15}, the minimum vertex cover problem~\cite{DBLP:conf/gecco/PourhassanGN15,DBLP:conf/ssci/PourhassanRN17,DBLP:conf/gecco/ShiNW18}, a dynamic constraint changes in the context of submodular optimization~\cite{DBLP:journals/corr/abs-1811-07806}.

We investigate the classical graph coloring problem that has already been studied in the context of evolutionary algorithms.
For the static problem, Fischer and Wegener~\cite{Fischer2005} considered a problem inspired by the Ising model from physics, where vertices of a graph need to be colored with the same color. On bipartite graphs, this corresponds to the classical graph coloring problem with 2 colors. They showed that on cycles, the \EA has expected optimization time $\Theta(n^3)$ under a reasonable assumption, but a simple (2+1)~Genetic Algorithm with 2-point crossover and fitness sharing succeeds in expected time $O(n^2)$.
Sudholt~\cite{Sudholt2005} considered the same problem on complete binary trees. He showed that, while ($\mu$+$\lambda$)~EAs take exponential expected time, the aforementioned  (2+1)~Genetic Algorithm finds an optimum in expected time $O(n^3)$.
Sutton~\cite{Sutton2016} presented bipartite graphs on which the (1+1)~EA needs superpolynomial time, with high probability. 
Sudholt and Zarges~\cite{Sudholt2010b} considered iterated local search algorithms in a different representation, where algorithms operate with an arbitrary number of colors, but the fitness function encourages the evolution of small color values. They considered mutation operators that can recolor large parts of a graph, based on so-called \emph{Kempe chains}. Along with a local search algorithm for graph coloring, iterated local search is shown to efficiently 2-color all bipartite graphs and to color all planar graphs with maximum degree at most~6 with at most~5 colors.
Recently, Bossek and Sudholt~\cite{Bossek2019a} also studied the performance of \EA and RLS for the edge coloring problem, where edges instead of vertices have to be colored such that no two incident edges share the same color, and the number of colors is minimized. 

Bossek \emph{et al.}~\cite{BNPS2019} considered a dynamic graph coloring problem where an edge is inserted into a properly colored graph. The authors analyze the expected time for the \EA, Randomized local search (RLS) and two iterated local search algorithms from~\cite{Sudholt2010b} to rediscover a proper coloring in case the newly added edge introduces a conflict. They consider 2-coloring bipartite graphs and 5-coloring planar graphs with maximum degree~6 as in~\cite{Sudholt2010b}. The authors show that dynamically adding an edge can lead to very hard symmetry problems that, in the worst case, may be harder to solve than coloring a graph  from scratch. On binary trees, RLS can easily get stuck in local optima and the \EA needs exponential expected time. 

\subsection{Our Contribution}
We consider the classical graph coloring problem and show that dynamic optimization can be helpful for this problem if the input graph is given to the algorithm incrementally based on an order determined by graph traversals.
Our investigations provide additional insights to a wide range of studies of evolutionary algorithms and other search heuristics that examine the computational complexity of these methods on instances of the graph coloring problem in static and dynamic environments. 

We consider an important aspect that bridges these static and dynamic studies to a certain extent. 
We are interested in whether giving an evolutionary algorithm the input graph in an incremental way and optimizing the resulting dynamic problem can lead to a faster optimization process than giving the algorithm the whole input at once as done in a standard static setting.
Our focus is on bipartite graphs, that is, the final graph resulting from the edge sequence is bipartite, which corresponds to the classical graph coloring problem with $2$ colors. This problem is polynomial time solvable in the context of problem specific algorithms. On the other hand, it is $\mathcal{NP}$-complete to decide if a given graph admits a $k$\nobreakdash-coloring for $k\geq 3$ \cite{garey1974some}. Furthermore, even if the input graph $G$ is promised to be $3$-colorable, it is $\mathcal{NP}$-hard to color $G$ with $4$ colors~\cite{guruswami2000hardness}.  %\pan{A bit discussions on $3$-colorable graphs.}

We examine a dynamic variant of the graph coloring problem in bipartite graphs where edges of a given static instance are made available to the algorithm over time. We show that, if the edges are provided in an order that preserves the connectivity of the graph, even the simple RLS can find proper colorings for all bipartite graphs efficiently. This is surprising since in the static setting, RLS fails badly even on simple bipartite graphs such as trees~\cite{BNPS2019}.
We further show that the order of edges is crucial: if edges are provided in a worst-case or random order, RLS only has an exponentially small probability of ever finding a proper 2-coloring on worst-case graph instances.
Specifically, we assume that the order in which the edges are made available is determined by a graph traversal algorithm. We study the reoptimization time after a given edge has created a conflict and show that the use of graph traversals leads to an efficient optimization process for a wide range of graph classes where evolutionary algorithms for the static setting (where the whole graph is given right at the beginning) fail. 
We pay special attention to popular graph traversal algorithms such as depth first search (DFS) and breadth-first-search (BFS) and show the difference that a choice between them may make with respect to the optimization time when carrying out dynamic graph coloring for bipartite graphs. 

Finally, we investigate speed ups that can be gained when using offspring populations and parallel dynamic reoptimization based on island models. We show that offspring populations of logarithmic size can decrease the expected optimization time by a linear factor. Island models that try to rediscover a proper coloring from the same initial coloring after adding an edge can benefit from independent evolution. It turns out that just using 3 islands leads to an asymptotically optimal runtime. This is one of very few examples where island models are proven to be more efficient than offspring populations and the first example where the speedup is not bounded in the number of islands.
Our results are summarized in Table~\ref{tab:all-times-incremental}.

The paper is structured as follows. In Section~\ref{sec2}, we introduce the graph coloring problem and the incremental reoptimization approaches that are subject to our analysis. In Section~\ref{sec:RLS_efficient_with_any_graph_traversal}, we show that RLS is efficient with any graph traversal, while Section~\ref{sec:graph-traversal-is-important} shows that not using graph traversals may be hugely inefficient. We carry out more detailed investigations when using BFS and DFS in Section~\ref{sec:choice_of_graph_traversal}. We show the benefit of using large enough offspring populations in Section~\ref{sec:offspring} and the benefit of parallel incremental reoptimization based on island models in Section~\ref{sec:island}. 

\begin{table}[htbp]
\caption{Worst-case expected times in the setting of adding edges incrementally to build up a whole bipartite graph for \unaware RLS (see Section~\ref{sec2}), \aware (1+$\lambda$)~RLS (see Section~\ref{sec:offspring}) and island models (see Section~\ref{sec:island}). We denote the length of the longest simple path by $\ell(G)$ and the diameter by $\diam(G)$.}
\label{tab:all-times-incremental}
\centering
%\vskip11pt
\renewcommand*{\arraystretch}{1.3}
\renewcommand*{\tabcolsep}{8pt}
\begin{scriptsize}
\begin{tabular}{lccc}
 %& \multicolumn{3}{c}{{\bf Incremental}} \\
%\cmidrule{3-5}
\toprule
{\bf Edge insertion order} & {\bf \unaware RLS} & \textbf{Tailored $(1 + \lambda)$ RLS} & \textbf{$\mu$ Islands} \\
\midrule
Any connectivity-preserving & $O(\ell(G)n^2 + m)$~\thmref{thm:incremental-rls-any-bipartite} & $O(\lambda m + \lambda 2^{-\lambda} \ell(G)n)$~\thmref{thm:upper-bound-for-offspring-populations} & $\Theta(m)$~\thmref{thm:island-model} \\
DFS traversal & $O(\ell(G)n^2 + m)$~\thmref{thm:incremental-rls-any-bipartite}  & $O(\lambda m + \lambda 2^{-\lambda} \ell(G)n)$~\thmref{thm:upper-bound-for-offspring-populations} & $\Theta(m)$~\thmref{thm:island-model} \\ 
BFS traversal & $O(\diam(G)n^2 + m)$~\thmref{thm:incremental-rls-bipartite-with-bfs} & $O(\lambda m + \lambda 2^{-\lambda} \mathrm{diam}(G)n)$~\thmref{thm:upper-bound-for-offspring-populations} & $\Theta(m)$~\thmref{thm:island-model} \\
\midrule
Random / worst-case insertion order & \multicolumn{3}{c}{$\infty$ (w.h.p.)~\thmref{thm:incremental-rls-worst-case-order}} \\
\bottomrule
\end{tabular}
\end{scriptsize}
\end{table}

\section{Preliminaries}
\label{sec2}
%The graph coloring problem is a classical combinatorial optimization problem.
Let $G = (V, E)$ denote an undirected graph with vertices~$V$ and
edges~$E$. We denote by $n := |V|$ the number of vertices and by $m := |E|$ the number of edges in $G$. We assume in the following that all considered graphs are connected (as otherwise connected components can be colored separately).
By $\ell(G)$ we denote the length of the longest simple path (number of edges) between any two vertices in the graph. The diameter $\mathrm{diam}(G)$ is the maximum number of edges on any \emph{shortest} path between any two vertices.

A \emph{vertex coloring} of $G$ is an assignment $c : V \to \{1, \ldots, n\}$ of color values to the vertices of~$G$.
Let $\deg(v)$ be the degree of a vertex~$v$ and $c(v)$ be its color in the current coloring.
Every edge $\{u, v\} \in E$ where $c(v) = c(u)$ is called a \emph{conflict}. A color is called \emph{free} for a vertex $v \in V$ if it is not assigned to any neighbor of $v$. The chromatic number $\chromatic(G)$ is the minimum number of colors that allows for a conflict-free coloring. A coloring is called \emph{proper} if there is no conflicting edge.

%In the following, we use $x, y \in \{1,\cdots,n\}^n$ to denote a vertex coloring of %$G$. 

% One can also define \textbf{edge coloring} similarly: each edge is colored a unique color, and an edge coloring is called be proper if there exists no two incident edges with the same color. (Two edges are called to be incident if they share one endpoint.)
% \pan{I think we do not need this in this paper, right?}

%\section{Two Settings}
%\pan{Can I change ``Infeasible Space'' to ``Bounded-Size Palette''? And ``feasible Space'' to ``Unbounded-Size Palette''? Note that even when we are talking about ``Infeasible Space'' in Section~\ref{subsec:infeasible_k2}, we are still looking for feasible coloring.}
%\subsection{Bounded-Size Palette}

We use the most common representation for graph coloring: the total number of colors is fixed and the objective function is to \emph{minimize the number of conflicts}. Since we only consider 2-coloring bipartite graphs, we can use the standard binary representation that assigns each vertex a color from $\{0, 1\}$. 
We use the notion of ``flipping'' vertices, by which we mean that the bit corresponding to the vertex' color is flipped.
%In the special case of $k=2$ this corresponds to binary strings, with a straightforward mapping between values $\{1, 2\}$ and $\{0, 1\}$, respectively.

The well-known randomized local search (RLS) is defined as follows. Assume that the current solution is $x$. In every iteration a single vertex color is flipped to produce $y$. Next, $x$ is replaced by $y$ if the fitness of $y$ is no worse than its parent fitness (see Algorithm~\ref{alg:rls}).
We consider all algorithms as infinite processes as we are mainly interested in the expected number of iterations until good solutions are found or rediscovered.
\begin{algorithm}[htb]
    \caption{RLS ($x$)}
    \algsetup{indent=1.5em}
    \begin{algorithmic}[1]
%    	\STATE Generate $x \in \{1, \dots, k\}^n$ uniformly at random.
        \WHILE{optimum not found}
        \STATE Generate $y$ by choosing an index $i \in \{1, \dots, n\}$ uniformly at random and flipping bit~$i$.
        \STATE If $y$ has no more conflicts than $x$, let $x := y$.
        \ENDWHILE
    \end{algorithmic}
    \label{alg:rls}
\end{algorithm}

Similar to~\cite{BNPS2019}, we also consider a \aware RLS algorithm that only mutates vertices that are involved in conflicts (see Algorithm~\ref{alg:multi-aware-rls}). We sometimes refer to the original RLS as \emph{\unaware RLS} as opposed to \aware RLS.

\begin{algorithm}[htb]
    \caption{\Aware RLS ($x$)}
    \algsetup{indent=1.5em}
    \begin{algorithmic}[1]
%    	\STATE Generate $x \in \{1, \dots, k\}^n$ uniformly at random.
        \WHILE{optimum not found}
        \STATE Generate $y$ by choosing a vertex $w$ uniformly at random from all vertices that are part of a conflict. Flip the color of~$w$. 
        \STATE If $y$ has no more conflicts than $x$, let $x := y$.
        \ENDWHILE
    \end{algorithmic}
    \label{alg:multi-aware-rls}
\end{algorithm}

% \frank{Check whether expected time is defined} 
% Analyzing the algorithms, we consider the expected number of fitness evaluations until a good coloring has been recomputed. We call this the expected time that an algorithm needs to achieve its goal. The expected optimization time refers to the expected number of fitness evaluations until an optimal solution is obtained.

We consider a setting of building up and re-optimizing a graph incrementally, a setting termed as \emph{incremental reoptimization} (IR) in the following. To be more precise, given a graph $G = (V, E)$ with $n$ nodes and $m$ edges, we start with an empty $n$-vertex graph $G' = (V, E')$ with $E' = \emptyset$ and assign colors to the nodes uniformly at random. Note, that $G'$ initially has no edges and hence no conflicts occur regardless of the colors assigned. Next, we subsequently add single edges to $E'$ according to a given order $\pi$ of the edges $e_1, \ldots, e_m \in E$, one by one, and re-optimize with algorithm $\mathcal{A}$, e.g., \unaware RLS, between edge insertions (see Algorithm~\ref{alg:increopt}).

\begin{algorithm}[htb]
    \caption{Incremental Reoptimization (IR) ($G \!=\! (V, \{e_1, \dots, e_m\})$, $\pi$, $\mathcal{A}$)}
    \algsetup{indent=1.5em}
    \begin{algorithmic}[1]
        \STATE Let $G' = (V, E')$ be a graph with $n=|V|$ isolated vertices ($E' = \emptyset$).
        \STATE Let $x$ be a coloring of all vertices, chosen uniformly at random.\!\!
        \FOR{$i = 1$ to $|E|$}
        \STATE Add edge $e_{\pi(i)}$ to $E'$.
        \STATE Run $\mathcal{A}$ on $G'$ with $x$ as the initial search point. Stop when a desired coloring has been obtained and store the final search point in $x$.
        \ENDFOR
    \end{algorithmic}
    \label{alg:increopt}
\end{algorithm}
\paragraph{\textbf{Graph traversal.}} Let $\pi$ be a sequence of edges $e_1,\cdots,e_m$ with endpoints in $V$. Let $G=(V,E)$ be the graph with $E=\{e_1,\cdots,e_m\}$. We will consider a special type of order $\pi$ that maximally preserves the connectivity. 
More precisely, for any $i\geq 1$, we let $G_i'$ be the \emph{edge-induced} subgraph of $G$ that is induced by the set of the first $i$ edges $\{e_1,\cdots,e_i\}$. That is, the edge set of $G'_i$ is $\{e_1,\cdots,e_i\}$ and the vertex set of $G_i'$ is the set of vertices that are endpoints of $e_j$, $1\leq j\leq i$. Note that $V(G_i')$ might be a strict subset of $V$. Now the order $\pi$ is called a \emph{graph traversal order} of $G$ if for any $i\geq 2$, the number of connected components (CCs) of $G_i'$ is at least the number of CCs of $G_{i-1}'$. In other words, an edge insertion can never link two CCs, which would reduce the number of CCs. Instead, the graph traversal needs to fully build one connected component before moving on to the next one.
Once an edge $\{u,v\}$ from some CC $C$ in $G$ appears, then the next edges gradually build a connected subgraph surrounding $u$ until all the edges in $C$ have appeared. After that, a different CC will be built, and so on. 
%\dirk{Err.. that means the number of connected components can never increase. This is not possible if the input graph has several CCs.}\pan{I think it was a typo, should be "at least" (rather than "at most")}

We call the order a \emph{Breadth-First-Search (BFS) traversal} or \emph{order}, if the ordering can be obtained by first selecting some starting vertex $v$ from each connected component, and then following edges in the same way that a breadth-first-search starting at $v$ would explore the connected component. A \emph{Depth-First-Search (DFS) traversal} or \emph{order} can be defined similarly except that  depth-first search is used. Note that both BFS and DFS traversal are special cases of graph traversal orders defined before.

\section{RLS is efficient with any graph traversal}
\label{sec:RLS_efficient_with_any_graph_traversal}

Our main research question is whether incremental optimization leads to efficient runtimes on subclasses of bipartite graphs if $\mathcal{A}$ is set to RLS. Recall that the worst-case expected time for discovering or re-discovering proper 2-colorings for bipartite graphs is infinite as demonstrated for binary trees in \cite{BNPS2019}. The key idea to prove the latter was to complete an $n$-vertex binary tree by adding a single edge which leads to strong symmetry problems if the linked parts are colored inversely.

It turns out that for IR in order to find proper 2-colorings of bipartite graphs efficiently, the order $\pi$ of edge insertions is crucial. This aspect will be further investigated in Section~\ref{sec:choice_of_graph_traversal}. For now we formulate the following general result:
\begin{theorem}
\label{thm:incremental-rls-any-bipartite}
Let $\ell(G)$ be the length of the longest path in~$G$. On every bipartite graph~$G$, the total expected time of IR with \unaware RLS to incrementally build a proper 2-coloring is at most $2n^2\ell(G) + m$ when edges are added in an order given by a graph traversal.
\end{theorem}

To prove Theorem~\ref{thm:incremental-rls-any-bipartite}, we make use of two folklore random walk results. The presentation is adapted from~\cite[Lemma~A.1]{Bossek2019a}.
\begin{lemma}%[Lemma~A.1 in~\cite{Bossek2019a}]
\label{lem:random-walk-tools}
Consider a fair random walk $X_t$ on $\{0, \dots, k\}$ where 0 is an absorbing state and $k$ is a reflecting state. More formally, abbreviating $p_{i, j} := \Prob{X_{t+1} = j \mid X_t = i}$, 
for all $0 < i < k$, $p_{i, i+1} = p_{i, i-1} = 1/2$, $p_{0, 0} = 1$ and $p_{k, k-1} = 1$.
Let $T_0$ be the first hitting time of state~0 and $T_{0, k}$ be the first hitting time of either state 0 or~$k$.
Then the following statements hold:
\begin{enumerate}
    \item For all $X_0$, $\E{T_0 \mid X_0} = X_0(2k-X_0-1)$.
    \item For all $X_0$ and all $r \in \mathbb{N}$, $\Prob{T_0 \ge 2rk^2 \mid X_0} \le 2^{-r}$.
    \item For all $X_0$, $\E{T_{0, k} \mid X_0} = X_0(k-X_0)$.
\end{enumerate}
All statements also hold for a lazy random walk with a self-loop probability of $1-p$, when multiplying all time bounds by~$1/p$.
\end{lemma}
\begin{proof}[Proof of Lemma~\ref{lem:random-walk-tools}]
The first two statements were shown in~\cite[Lemma~A.1]{Bossek2019a}. The third statement follows from the fair gambler's ruin scenario where one player starts with $X_0$ dollars and the other player starts with $k-X_0$ dollars and the game ends when either player is broke. It is well known that the expected time for the game to end is $X_0(k-X_0)$.
\end{proof}

\begin{proof}[Proof of Theorem~\ref{thm:incremental-rls-any-bipartite}]
Note that in our setting we start with an $n$-vertex graph with no edges at all, each vertex having color 0 or 1 with equal probability. Now we add edges incrementally in an order of a graph traversal. Since the graph is bipartite, adding a single edge $e = \{u, v\}$ links two vertices of different sets. This step may introduce at most one conflict if $c(u) \neq c(v)$. Note that this can happen only if one vertex, w.\,l.\,o.\,g.\ $v$, has degree one after insertion of $e$, i.e., $v$ has not yet been linked to the growing connected component before. Otherwise, $e$ closes a cycle~$C$. This cycle must be of even length since the graph is bipartite and the path $C \setminus \{e\}$ has alternating colors since the previous coloring was proper. Thus $u$ and $v$ must have different colors already. In this case, inserting $e$ does not create a conflict.

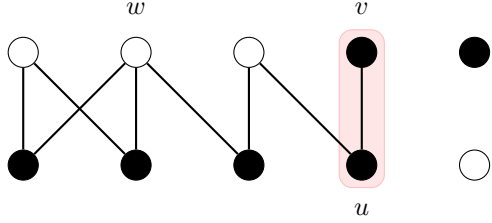
\begin{figure}
\centering
\begin{tikzpicture}[scale=1.5]

% invisible dummy nodes
\node (dummy1) at (3,0) {};
\node (dummy2) at (3,1) {};
\node[rounded corners=5pt, inner sep=5pt, draw=red!25, fill=red!10, fit=(dummy1) (dummy2)] {};

% U = V_1
\node[vertex, fill=black] (u1) at (0,0) {};
\node[vertex, fill=black] (u2) at (1,0) {};
\node[vertex, fill=black] (u3) at (2,0) {};
\node[vertex, fill=black, label={[label distance=5pt]below:{$u$}}] (u4) at (3,0) {};
\node[vertex, fill=white] (u5) at (4,0) {};

% W = V_2
\node[vertex] (w1) at (0,1) {};
\node[vertex, label={[label distance=5pt]above:{$w$}}] (w2) at (1,1) {};
\node[vertex] (w3) at (2,1) {};
\node[vertex, fill=black, label={[label distance=5pt]above:{$v$}}] (w4) at (3,1) {};
\node[vertex, fill=black] (w5) at (4,1) {};

% Now for the edges
\draw[thick] (u1) -- (w1) -- (u2) -- (w2) -- (u3) -- (w3) -- (u4);
\draw[thick] (u1) -- (w2);
% conflict edge
\draw (u4) edge[thick] (w4); %node[right] {$e$}

\end{tikzpicture}
\caption{Snapshot of an IR iteration where a random walk might take place. Here, edge $\{u,v\}$ was added last in the course of incremental optimization and lead to a single conflict. Mutating $v$ resolves the conflict while mutating $u$ moves the conflict to the other edge incident with $u$. The conflict can then propagate further to the left where node $w$ serves as a reflecting node for the random walk.}
\label{fig:random-walk-example}
\end{figure}

Now, assume there is a conflict $\{u, v\}$ and let $v$ be the vertex with degree~$1$. Mutating~$v$ will resolve the conflict. However, if $u$ has degree~2, mutating~$u$ moves the conflict to the other incident edge at~$u$ (see Fig.~\ref{fig:random-walk-example} for an illustration).
This yields a random walk that can be mapped to the integers as follows. Let $d(u, v)$ be the graph distance, that is the smallest number of edges on any path between $u$ and $v$. 
If the conflict involves an edge $\{v_1, v_2\}$ then the current state is defined as
\[
1 + \min\{d(v_1, v), d(v_2, v)\}
\]
with an additional absorbing state~0 that is attained when the conflict is resolved. 
The random walk always starts in state~1 as initially $\{u, v\}$ is the conflicting edge.
The random walk is fair since flipping the vertex that is closer to~$v$ decreases the state by 1, and flipping the other vertex increases it by~1, if this mutation is accepted. It is accepted if and only if the mutated vertex has degree at most~2 as otherwise the number of conflicts increases. Hence, the random walk is reflected at the first vertex on the path from $v$ that has degree greater than~2; if there is no such vertex, there is another leaf at which the conflict can be resolved. The maximum state that can be reached is bounded by $\ell(G)$, i.\,e. the length of the longest path in~$G$ (since the closest vertex to~$v$ must have graph distance at most $\ell(G)-1$). This random walk requires at most $2\ell(G)$ relevant steps by Lemma~\ref{lem:random-walk-tools}. Each propagating step happens with probability at least $1/n$ and thus has waiting time $O(n)$. 

Finally, recall that every time an edge insertion closes a cycle no conflict is introduced at all as argued at the beginning of the proof. In these cases $\mathcal{A}$ terminates after a single fitness function evaluation. As a consequence, only the cases where an isolated node is linked for the first time may introduce a conflict. There are $n-1$ such steps. Hence the total runtime is
\begin{align*}
    2(n-1)n\ell(G) + (m-n+1) \le 2n^2\ell(G) + m. \qquad \qedhere
\end{align*}
\end{proof}

The upper bound from Theorem~\ref{thm:incremental-rls-any-bipartite} is tight on path graphs.

\begin{theorem}
\label{thm:incremental-rls-on-n-vertex-path}
On any path with $n$ nodes, the total expected time of IR with \unaware RLS to incrementally build a proper 2-coloring is $\Omega(n^3)$ when edges are added in an order given by a graph traversal.
\end{theorem}

\begin{proof}
Consider an $n$-vertex path which is built incrementally starting from either one of its leaf nodes. After adding the $i$-th edge $e = \{u, v\}$, $1 \leq i \leq n-1$, with probability $1/2$ no conflict is introduced if by chance $c(u) \neq c(v)$. With the converse probability, if $u$ and $v$ have the same colors, a random walk with states $0, \dots, i$ is started, where both states $0$ and $i$ are goal states and the random walk starts in state~1. This random walk runs for at least $i-1$ relevant steps in expectation by Lemma~\ref{lem:random-walk-tools} and a relevant step happens with probability at most $2/n$. In total we add $n-1$ edges incrementally and all $n-1$ events of a random walk taking place are independent. There are $(n-1)/2$ such random walks in expectation. Note that, by Chernoff bound, the probability of having less than $(n-1)/3$ random walks is $1-e^{-\Omega(n)}$.
Let $j_1, \ldots, j_{(n-1)/3}$ be the steps a random walk takes place and note that $j_i \ge i$.
Then the expected time to incrementally reoptimize a path is bounded from below by
\begin{align*}
    \sum_{i=1}^{(n-1)/3} \frac{n (i-1)}{2} = \frac{n}{2} \sum_{i=1}^{(n-1)/3} (i-1) = \Omega(n^3).
\end{align*}
Here, the first term results from the fact that the length of the random walks is monotonically increasing. Note that $\Omega(n^3) = \Omega(mn^2)$ for paths since $m = \Theta(n)$.
\end{proof}

\begin{figure}[htb]
\centering
\begin{tikzpicture}[scale=1]
    \node[vertex] (v0) at (0,0) {};

    \node[vertex] (v1) at (-3,1) {};
    \node[vertex] (v2) at (-2,1) {};
    \node[vertex] (v3) at (-1,1) {};
    \node[vertex] (v4) at (1,1) {};
    \node[vertex] (v5) at (2,1) {};
    \node[vertex] (v6) at (3,1) {};
    \node[vertex] (v7) at (-3,0) {};
    \node[vertex] (v8) at (-2,0) {};
    \node[vertex] (v9) at (-1,0) {};
    \node[vertex] (v10) at (1,0) {};
    \node[vertex] (v11) at (2,0) {};
    \node[vertex] (v12) at (3,0) {};
    \node[vertex] (v13) at (-3,-1) {};
    \node[vertex] (v14) at (-2,-1) {};
    \node[vertex] (v15) at (-1,-1) {};
    \node[vertex] (v16) at (1,-1) {};
    \node[vertex] (v17) at (2,-1) {};
    \node[vertex] (v18) at (3,-1) {};
    
    \foreach \from/\to in {1/2, 2/3, 4/5, 5/6, 7/8, 8/9, 10/11, 11/12, 13/14, 14/15, 16/17, 17/18, 3/0, 4/0, 9/0, 10/0, 15/0, 16/0}{
        \draw (v\from) edge[solid] (v\to);
    }
        
    % \node[circle, minimum width=1.5cm, draw=white] (C1) {};
    % \node[circle, minimum width=4cm, draw=white] (C2) {};
    % \node[vertex] (center) {};
    % \foreach \i in {1,2,3,4,5}{
    %     \pgfmathsetmacro{\angle}{(\i)*360/5};
    %     \pgfmathsetmacro{\j}{\i+5};
    %     \node[vertex] (v\i) at (C2.\angle) {};
    %     \draw (center) edge[solid] (v\i);
    %     \draw (v\i) edge[solid] (v\j); % no idea why this is not working
    % }
\end{tikzpicture}
\caption{Example of a depth-$k$ star with $n = 19$ nodes and depth $k = 3$.}
\label{fig:depth-k-star}
\end{figure}

Paths are examples where the upper bound from Theorem~\ref{thm:incremental-rls-any-bipartite} is tight for a maximum value of $\ell(G)$, namely $\ell(G)=n$. We also show that there is a family of graphs for all (even) values of $\ell(G)$ for which the upper bound from Theorem~\ref{thm:incremental-rls-any-bipartite} is tight. Consider a generalization of the star-graph termed the \emph{depth-$k$ star} where we have one center node and $(n-1)/k$ paths originating in the center node (see Fig.~\ref{fig:depth-k-star} for an example), for some value $1 \le k \leq (n-1)/3$. (For simplicity we assume that $(n-1)/k$ is integer). Note that only the center node can have a degree greater $2$ and may serve as a reflecting node in the course of incremental optimization. Hence, the behavior of RLS is similar to its behavior on a path. In the following we show that the runtime bound from Theorem~\ref{thm:incremental-rls-any-bipartite} is tight on depth-$k$ stars for any reasonable choice of $k$.

\begin{theorem}
\label{thm:incremental-rls-depth-k-star}
On any depth-$k$ star with with $n$ nodes ($n$ odd) and $1 \leq k \leq (n-1)/3$, the expected time of \unaware RLS to build a proper 2-coloring is $\Omega(kn^2)$ when edges are added in an order given by a graph traversal.
\end{theorem}

\begin{proof}
Note that $\ell(G) = 2k$ (as each path from one leaf to another is a longest path) for any depth-$k$ star. Note further that the center node reflects random walks once it reaches a vertex of degree~$3$ in the course of incremental optimization. This must happen after adding $2k+1$ edges since edges are added according to a graph traversal and the center node is the only link between paths. At the time a third edge at the center is added, there can only be two paths that have been built, or partially build. We consider the expected remaining time for adding the remaining $(n-1)/k-2$ paths. Note that for all these paths, the addition of edges must start from the center vertex and now the center node acts as a reflecting node for these random walks.
%from this event as then all random walks are limited to a length of $k$. 

%Consider a depth-$k$ star which is optimized incrementally by IR starting from one of its leaf nodes. In the following we pessimistically assume that in the first $2k$ iterations consecutive edge insertions form a path from the starting leaf node to another leaf node. Note that by the above argument this pessimistic approach does not affect the asymptotic length of random walks.
%This path has length $\ell(G) = 2k$ and hence, following the argument in Theorem~\ref{thm:incremental-rls-on-n-vertex-path}, the expected time to color it properly is $\Omega(nk^2)$. Next, all other paths of length $k$ originating in the center node are formed following the graph traversal. However, now the center node acts as a reflecting node for a possible random walk.
% Note further that each branch is colored independently of the other branches due to the reflection in the center node.
After adding the $i$-th edge $e=\{u,v\}$ of a path, with probability $1/2$ a random walk with states $0,\ldots,i$ starts. By Lemma~\ref{lem:random-walk-tools} this random walk runs for at least $2(i-1)$ steps in expectation and relevant steps take place with probability at most $2/n$. For a fixed path in total $k$ edges are added until a leaf node is connected to the growing connected component. 
Let $T_j$ be the number of generations spent fixing a conflict on the $j$-th path, then 
\[
\E{T_j} \ge \sum_{i=1}^{k} \frac{1}{2} \cdot 2(i-1) = \Omega(nk^2).
\]

% There are $k/2$ random walks in expectation for each path with a probability of $e^{-\Omega(k)}$ for less than $k/3$ random walks taking place by classical Chernoff bound. Let $j_1, \ldots,j_{k/3}$ \pan{$j_i$ never used afterwards?} be the steps a random walk takes place. Then the expected time to reoptimize a single path is at least
% \begin{align*}
%     \sum_{i=1}^{k/3} \frac{n2(i-1)}{2} = n \sum_{i=1}^{k/3} (i-1) = \Omega(nk^2).
% \end{align*}
By construction of the depth-$k$ star there are $(n-1)/k$ paths and two of these were covered in the first phase. Adding up all times spent on the remaining paths, the expected number of steps until the depth-$k$ star is properly colored with two colors is
\begin{align*}
    \sum_{j=1}^{(n-1)/k - 2} \E{T_j} = \left(\frac{(n-1)}{k} - 2\right)\Omega(nk^2) = \Omega(kn^2).\qquad  \qedhere
\end{align*}
\end{proof}

Recall that $\ell(G)=2k=\Theta(k)$. As a consequence, the runtime of IR with \unaware RLS with any graph traversal on any depth-$k$ star is tight for any valid choice of the graph parameter $k$. 

% \jakob{Needs polishing}
% Note that for each such graph $G$ we have $\ell(G) = \diam(G) = 2k+1$.
% I believe we can argue that RLS has expected runtime $\Omega(k^2n)$ on every single path. There are $(n-1)/k$ such path. Hence the running time is at least
% \begin{align*}
%     \frac{(n-1)}{k} nk^2 = \Omega(kn^2).
% \end{align*}
% Recalling $\ell(G) = \diam(G) = \Theta(k)$ we obtain the result.

We finish this section by noting that, similarly to~\cite{BNPS2019}, the expected runtime can be reduced by using \aware RLS which reduces the waiting time for re-coloring the right vertex from $O(n)$ to $\Theta(1)$.

\begin{corollary}
\label{thm:incremental-aware-rls-any-bipartite}
On any bipartite graph, the total expected time of \aware RLS to incrementally build a proper 2-coloring is $O(\ell({G})n + m)$ when edges are added in an order given by a graph traversal.
\end{corollary}

% \begin{proof}
% The proof is identical to the proof of Theorem~\ref{thm:incremental-rls-any-bipartite}. The only argument to modify is the waiting time for re-coloring of the right vertex. In the \aware setting the waiting time is $O(1)$.
% \end{proof}

\section{Graph Traversals are Important}
\label{sec:graph-traversal-is-important}

The following result emphasizes that the order of edge insertions is of utmost importance; an unfavorable order may lead to infinite runtimes for RLS with overwhelming probability. Furthermore, even if the order is uniformly random, it may still lead to infinite runtimes for RLS. Given a graph $G=(V,E)$, and an edge sequence $\pi$ over $E$, we say $\pi$ is a random order of $E$ or the graph if $\pi$ is chosen uniformly at random from the set of all possible permutations over~$E$.

\begin{theorem}
\label{thm:incremental-rls-worst-case-order}
For every $n = 1 \bmod 3$ there exists a tree $T_n$ and a worst-case edge insertion strategy such that RLS has infinite runtime with probability $1 - 2^{-\Omega(n)}$. Furthermore, for the random order of $T_n$, RLS has infinite runtime with probability $1 - 2^{-\Omega(n)}$.
\end{theorem}

\begin{proof}
We consider a tree $T_n$ where the root $r$ has $(n-1)/3$ children and each child of the root has two children. This means that on level 1 of $T_n$, we have $(n-1)/3$ binary trees of height~1. Now consider the following worst-case edge insertion strategy: first add edges such that all $(n-1)/3$ binary trees are formed (phase~1) and afterwards connect the root to its children (phase~2). Note that once two binary trees are colored inversely, RLS gets stuck forever since there is no possibility to color $r$ without conflicts after connecting both binary trees to the root. This is because the root's children -- once connected to the root -- have degree greater 2 and thus act as reflecting states for the random walk of the introduced conflict. Since in the first phase of the edge insertion all binary trees are unconnected and hence colored independently, the probability that they are all colored the same is $2 \cdot 2^{-(n-1)/3} = 2^{-\Omega(n)}$. Hence, the unfavorable situation occurs with probability $1 - 2^{-\Omega(n)}$.

Finally, if the edges are inserted in random order, i.e., the edge sequence is chosen uniformly at random from the set of all edge permutations over $E$, then for each height-$1$ binary tree with the vertex $v$ being the child of root $r$, the probability that both edges in the tree appear first before edge $(v,r)$ is $\frac{1}{3}$. We call a height\nobreakdash-$1$ binary tree \emph{bad} if both of its edges appear before the edge connecting $r$ to its child in the tree. Therefore, the expected number of bad binary trees is $\frac{1}{3} \cdot \frac{n-1}{3}=\frac{n-1}{9}$. Further note that all the bad binary trees occur independently due to the random order assumption. By Chernoff bound, with probability at least $1-2^{-\Omega(n)}$, the number of bad binary trees is at least $T\geq \frac12 \cdot \frac{n-1}{9}$. Finally, for all these bad binary trees, since they are unconnected to the rest of the tree when they are formed, and hence colored independently, the probability that they are all the same is $2^{-\Omega(T)}=2^{-\Omega(n)}$. Hence, the unfavorable situation occurs with probability $1-2^{-\Omega(n)}-2^{-\Omega(n)}=1-2^{-\Omega(n)}$.
\end{proof}
%Note that this particular instance class can be solved efficiently by the \EA as it can invert the color of a binary tree with a specific 3-bit mutation. This strategy, however, fails on deeper trees~\cite{Sudholt2005}. 

\section{On the choice of graph traversal}
\label{sec:choice_of_graph_traversal}

Theorem~\ref{thm:incremental-rls-any-bipartite} states that (generic) RLS is efficient with any con\-nec\-tiv\-i\-ty-preserving graph traversal. In the following we study the effect of using DFS- versus BFS-traversals and point out major differences on special cases of bipartite graphs.
To motivate this, consider a complete bipartite graph $G = (V_1 \cup V_2, E)$ with $n_1 = |V_1|$, $n_2 = |V_2|$ and $n_1 = n_2 = n/2$. Note that given an arbitrary starting node $v \in V_1$ there is a DFS-traversal that adds edges in an order such that after adding the first $n-1$ edges, the partial graph is an $n$-vertex path. Such a DFS-traversal can be easily constructed by following an edge to a node that was not yet connected to the growing connected component. This path has length $\Theta(n)$ and is a longest path in $G$, i.e., $\ell(G) = \Theta(n)$. Now consider a BFS-traversal and assume w.\,l.\,o.\,g.\ that we start in an arbitrary node $v \in V_1$. Now, according to the working principles of BFS, BFS adds all $n/2$ edges to the neighbors of $v_1$, first producing random walks of length at most~2 in the optimization steps of IR. Subsequently, for each vertex in~$V_2$, all $n_1-1$ edges to the remaining nodes in $V_1$ are added. Again, each IR step deals with random walks of length at most $3$. Hence, the length of the paths introduced by BFS is $\Theta(1)$ vs. $\Theta(n)$ for DFS (see Fig.~\ref{fig:complete-bipartite-dfs-and-bfs-traversals} for an illustration).

\begin{figure}
\centering
\begin{tikzpicture}[scale=1.13]
% DFS-traversal
\foreach \x/\y/\nr/\order in {0/0/1/, 0/1/2/1, 1/0/3/2, 1/1/4/3, 2/0/5/4, 2/1/6/5, 3/0/7/6, 3/1/8/7} {
    \node[vertex] (v\nr) at (\x,\y) {$\order$};
}
\draw (v1) -- (v2) -- (v3) -- (v4) -- (v5) -- (v6) -- (v7) -- (v8);
\end{tikzpicture}
\hspace{0.5cm}
\begin{tikzpicture}[scale=1.13]
% BFS-traversal
\foreach \x/\y/\nr/\order in {0/0/1/, 0/1/2/1, 1/0/3/5, 1/1/4/2, 2/0/5/6, 2/1/6/3, 3/0/7/7, 3/1/8/4} {
    \node[vertex] (v\nr) at (\x,\y) {$\order$};
}
\draw (v1) -- (v2);
\draw (v1) -- (v4);
\draw (v1) -- (v6);
\draw (v1) -- (v8);

\draw (v2) -- (v3);
\draw (v2) -- (v5);
\draw (v2) -- (v7);
\end{tikzpicture}
\caption{Example of possible first $n-1$ edge insertions following a DFS traversal (left) and BFS traversal (right) on a complete bipartite graph. Nodes are numbered with the iteration they are linked to the growing connected component. For visual clarity all other edges are not shown.}
\label{fig:complete-bipartite-dfs-and-bfs-traversals}
\end{figure}
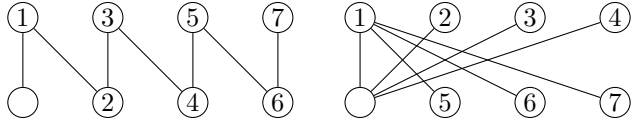

Since BFS visits the nodes in level order, level by level, we can substitute $\ell(G)$ with the diameter of the graph~$G$, denoted by $\diam(G)$, in the expected runtime bound. This observation is made mathematically rigorous in the following theorem.

\begin{theorem}
\label{thm:incremental-rls-bipartite-with-bfs}
On any bipartite graph, the total expected time of IR with \unaware RLS to incrementally build a proper 2-coloring is $O(\diam(G)n^2 + m)$ when edges are added in order of a breadth-first-search traversal.
\end{theorem}

\begin{proof}
We focus on the maximum length of random walks that may occur during the optimization. First of all note that BFS traverses a graph in level order visiting all adjacent nodes first, nodes with distance two second and so on. Put differently, BFS solves the unweighted Single-Source-Shortest-Path (USSSP) problem. That is, given a starting node $u \in V$, the length of each path in a BFS-traversal until a previously seen node is visited again is bounded by the length of the longest shortest path $s(G, u)$ to any other vertex $v \in V \setminus \{u\}$ -- in terms of the number of edges on the path. Since $s(G, u)$ depends on the starting node, the length of the longest possible path produced by incrementally adding edges by any BFS traversal is upper bounded by the diameter $\diam(G) = \max_{u \in V} s(G, u)$, i.~e., the length of the longest shortest path in $G$. Adopting the waiting-time arguments of Theorem~\ref{thm:incremental-rls-any-bipartite} we obtain a runtime bound of $O(\diam(G)n^2 + m)$ for any BFS-traversal.
\end{proof}

This bound is tight on paths and depth-$k$ stars as for both graph classes $\diam(G)=\ell(G)$.

Even though the asymptotic runtime bounds are the same, e.g. on paths, it makes a huge difference for other sub-classes of bipartite graphs. As pointed out in the beginning of this section, on complete bipartite graphs $\ell(G) = \Theta(n)$ whereas $\diam(G) = \Theta(1)$, yielding a performance advantage of a factor of $n$ for BFS traversals. Similarly, on toroids, $\ell(G) = \Theta(n)$ and $\diam(G) = \Theta(\sqrt{n})$. 
As $\diam(G) \leq \ell(G)$ on any graph there is no advantage of using DFS and the usage of BFS shows similar or superior performance. Table~\ref{tab:dfs-vs-bfs-on-special-bipartite} gives an overview of the expected runtimes of RLS with DFS and BFS on sub-classes of bipartite graphs as well as further results obtained in the following sections.

\begin{table}[htbp]
\caption{Obtained runtime results for sub-classes of bipartite graphs. For $k$-ary trees we assume that $2 \le k = O(1)$. The toroid is assumed to have lengths $\sqrt{n} \times \sqrt{n}$. For \lRLS we use the choice of~$\lambda = \lambda^* := \max\{\lceil \log(\mathrm{diam}(G)n/m\rceil, 1\}$. The island model uses the optimal number of 3 islands. We use $\Theta$ where we have an explicit lower bound or the trivial one of $\Omega(m)$. Note that the island model has optimal performance $\Theta(m)$ on all bipartite graph classes.}
\label{tab:dfs-vs-bfs-on-special-bipartite}
\centering
%\vskip11pt
\renewcommand*{\arraystretch}{1.3}
\renewcommand*{\tabcolsep}{6pt}
\begin{scriptsize}
\begin{tabular}{lccccccc}
%\multicolumn{3}{c}{} & \multicolumn{3}{c}{{\bf Incremental}} \\
%\cmidrule{4-6}
\toprule
\textbf{Graph class} & $\ell(G)$ & $\diam(G)$ & \textbf{RLS with DFS} & \textbf{RLS with BFS} & \textbf{\makecell{Tailored\\ RLS with BFS}} & \textbf{\makecell{Tailored\\ $(1+\lambda)$ RLS with BFS}} & \textbf{\makecell{Island \\Model}} \\
\midrule
Complete $k$-ary tree & $\Theta(\log n)$ & $\Theta(\log n)$ & $O(n^2 \log n)$ & $O(n^2 \log n)$ & $O(n \log n)$ & $O(n \log \log n)$ & $\Theta(n)$\\
Toroid & $\Theta(n)$ & $\Theta(\sqrt{n})$ & $O(n^3)$ & $O(n^{5/2})$ & $O(n^{3/2})$ & $O(n \log n)$ & $\Theta(n)$ \\
$(\log n)$-dim.\ hypercube & $\Theta(n)$ & $\Theta(\log n)$ & $O(n^3)$ & $O(n^2 \log n)$ & $\Theta(n \log n)$ & $\Theta(n \log n)$ & $\Theta(n \log n)$ \\
Path & $\Theta(n)$ & $\Theta(n)$ & $\Theta(n^3)$ & $\Theta(n^3)$ & $\Theta(n^2)$ & $\Theta(n \log n)$ & $\Theta(n)$ \\
Star graph & $\Theta(1)$ & $\Theta(1)$ & $\Theta(n^2)$ & $\Theta(n^2)$ & $\Theta(n)$ & $\Theta(n)$ & $\Theta(n)$ \\
Complete bipartite & $\Theta(n)$ & $\Theta(1)$ & $O(n^3)$ & $\Theta(n^2)$ & $\Theta(n^2)$ & $\Theta(n^2)$ & $\Theta(n^2)$ \\
Depth-$k$ star & $\Theta(k)$ & $\Theta(k)$ & $\Theta(kn^2)$ & $\Theta(kn^2)$ & $\Theta(kn)$ & $O(n \log k)$ & $\Theta(n)$ \\
\bottomrule
\end{tabular}
\end{scriptsize}
\end{table}

For sake of completeness we close this section with a corollary on the runtime of IR with \aware RLS and BFS.
\begin{corollary}
\label{thm:incremental-conflict-aware-rls-complete-bipartite}
On any complete bipartite graph, the total expected time of IR with \aware RLS to incrementally build a proper 2-coloring is $O(\diam(G)n + m)$ when edges are added in order of a breadth-first-search traversal.
\end{corollary}

% \begin{proof}
% In this setting the waiting time to find the conflicting vertices is only of order $O(1)$. The proof then follows from the proof of Theorem~\ref{thm:incremental-rls-any-bipartite}.
% \end{proof}

\section{Offspring Populations}
\label{sec:offspring}

We now consider the use of offspring populations in RLS. The \lRLS creates $\lambda$ offspring through independent mutations from the current search point, and then picks a best offspring that is compared against the parent as in RLS. Ties between offspring are broken uniformly at random. For simplicity, we only consider tailored RLS in the following, but it easy to derive bounds on generic RLS with offspring populations.
The following theorem quantifies the improved time bounds when using BFS and DFS.
\begin{theorem}
\label{thm:upper-bound-for-offspring-populations}
For a given connected graph~$G$, let $L(G)$ denote an upper bound on the length of any random walk; more specifically, $L := \diam(G)$ when using BFS and $L := \ell(G)$ for any other graph traversal. 
Then the expected time of tailored \lRLS is $O(\lambda m + \lambda 2^{-\lambda} L n)$.

For $\lambda^* := \max\{\lceil \log (Ln/m) \rceil, 1\}$ this is $O(\lambda^* m)$.
\end{theorem}
\begin{proof}
Consider the situation after adding one edge, which leads to a conflict. The conflict is resolved in one generation if there is an offspring that flips the leaf node. This happens with probability $1-2^{-\lambda}$. With the converse probability $2^{-\lambda}$, all offspring flipped the leaf's neighbor and the conflict moved away from the added edge.

We argue that, while both end points of the conflicting edge have degree at least 2, \lRLS behaves like RLS. Assume both end points have degree~2. Since there is no way of resolving the conflict in one step, all offspring will have the same fitness. Since all offspring are generated independently and with identical distributions, we may assume w.\,l.\,o.\,g.\ that the first offspring is selected for survival. This means that the remaining offspring are irrelevant and \lRLS simulates a step of RLS. 
If one end point of the conflicting edge has degree larger than~2, flipping this end point leads to an offspring with a worse fitness. Hence the only accepted step is to flip the edge's other end point. Having multiple offspring can only decrease the time until this step happens.

Using our upper bound on RLS (Theorem~\ref{thm:incremental-rls-any-bipartite}), \lRLS resolves the conflict after any edge insertion after $O(1 + 2^{-\lambda} L)$ generations. Since one generation creates $\lambda$ evaluations, the number of evaluations is $O(\lambda + 2^{-\lambda} \lambda L)$. Since we only have at most $n$ random walks, the total time for solving random walks is $O(\lambda n + 2^{-\lambda} \lambda Ln)$. Iterations where no random walks are necessary make $\lambda$ evaluations. Together, this yields an upper bound of $\lambda m + O(\lambda 2^{-\lambda} Ln)$. 

For $\lambda^* = \max\{\lceil \log(Ln/m) \rceil, 1\}$, the last term simplifies to $O(\lambda^* m)$ if $\log(Ln/m) \ge 1$, or equivalently, $Ln \ge m$. Otherwise, the bound is dominated by the first term $O(\lambda^*m)$.
\end{proof}

For paths the upper bound from Theorem~\ref{thm:upper-bound-for-offspring-populations} is tight.
\begin{theorem}
\label{thm:lower-bound-for-offspring-populations}
The expected reoptimization time of tailored \lRLS on a path with any graph traversal is $\Omega(\lambda m + \lambda 2^{-\lambda} n^2)$.
\end{theorem}
\begin{proof}
The proof is similar to the lower bound for RLS on paths (Theorem~\ref{thm:incremental-rls-on-n-vertex-path}). 
Consider a random walk started after inserting the $i$-th edge. Recall that the random walk has states $0, \dots, i$ and both states 0 and~$i$ are goal states. Whenever the state of the random walk is $1$ or $i-1$, there is a probability of $1-2^{-\lambda}$ that one of the offspring finds a goal state. As argued in the proof of Theorem~\ref{thm:upper-bound-for-offspring-populations}, on states $2, \dots, i-2$ the \lRLS behaves like RLS. Hence, with probability $2^{-\lambda}$, state~2 is reached after the first generation and then \lRLS needs at least $i-3$ relevant steps in expectation to reach either state~1 or state~$i-1$. If this happens, we assume pessimistically that a proper coloring is found. Summing up expected times as in the proof of Theorem~\ref{thm:incremental-rls-on-n-vertex-path} implies the claim.
\end{proof}

\section{Island Models}
\label{sec:island}

We now consider island models that evolve several populations in parallel and communicate to exchange good solutions. More specifically, at each step of the IR process, there exist $\lambda$ islands that each run a tailored RLS. All islands are all started on the same graph after inserting a new edge, with the same initial coloring. The islands run independently until the first island has found a proper coloring; then the proper coloring is shared with all islands (ties broken arbitrarily but ensuring that all islands store the same proper coloring). Note that we implicitly use a complete graph as migration topology (though our main result applies to all topologies containing a triangle).  Algorithm~\ref{alg:increopt-islands} shows the respective pseudocode.

\begin{algorithm}[htb]
    \caption{Incremental Reoptimization (IR) ($G \!=\! (V, \{e_1, \dots, e_m\})$, $\pi$, $\mathcal{A}$) using an island model}
    \algsetup{indent=1.5em}
    \begin{algorithmic}[1]
        \STATE Let $G' = (V, E')$ be a graph with $n=|V|$ isolated vertices ($E' = \emptyset$).
        \STATE Let $x$ be a coloring of all vertices chosen uniformly at random.
        \FOR{$i = 1$ to $|E|$}
        \STATE Add edge $e_{\pi(i)}$ to $E'$.
        \STATE Run $\lambda$ tailored RLSs on $G'$ with $x$ as the initial search point. In every generation, check whether an island has obtained a desired coloring. If so, store the final search point in $x$.
        \ENDFOR
    \end{algorithmic}
    \label{alg:increopt-islands}
\end{algorithm}

We will show that independent evolution steps are more efficient than offspring populations. Our main result in this section is:
\begin{theorem}
\label{thm:island-model}
For any graph traversal order, the expected reoptimization time of the island model is $\Theta(\lambda m)$ for $\lambda \ge 3$. For $\lambda=3$ we get an optimal time of $\Theta(m)$.
\end{theorem}
The surprising finding is that 3 islands are sufficient to obtain an asymptotically optimal reoptimization time. This is one of very few examples where island models perform better than offspring populations. The only other examples we are aware of in the context of rigorous runtime analysis are an artificially constructed function~\cite{Lassig2013} and a particular instance for the Eulerian Cycle problem~\cite{Lassig2014}. In the latter case, the speedup is exponential in~$\lambda$. To our knowledge, Theorem~\ref{thm:island-model} gives the first example where the speedup is not bounded by a function of~$\lambda$.

To prove Theorem~\ref{thm:island-model}, we first study independent fair random walks and analyze the time until \emph{the first} random walk reaches the target state. 
The following lemma may be of independent interest.
\begin{lemma}
\label{lem:independent-random-walks}
Consider $\eta$ independent random walks as defined in Lemma~\ref{lem:random-walk-tools}. Let $T_{\eta}$ be the first point in time any of the $\eta$ random walks reaches state~0, assuming that all random walks start in state~1. Then
\begin{enumerate}
    \item There is a constant $c > 0$ such that $\Prob{T_\eta \ge t} \le c^{-\eta} t^{-\eta/2}$.
    \item $\E{T_\eta} = \begin{cases}
        2k-2 & \eta=1\\
        \Theta(\log k) & \eta=2\\
        O(1) & \eta \ge 3
        \end{cases}$
\end{enumerate}
\end{lemma}
\begin{proof}%[Proof of Lemma~\ref{lem:independent-random-walks}]
We first consider a single random walk, that is, $\eta=1$. Here the claim on the expectation follows from folklore argument, formalised in the first statement of Lemma~\ref{lem:random-walk-tools}.

% According to~\cite[III.7, Theorem~2]{Feller1}, the probability that state~0 will be hit at time~$t$ is
% \[
% \Prob{T_1 = t} = 
% \frac{1}{t} \binom{t}{\frac{t + 1}{2}} \cdot 2^{-t}
% \]
% where the binomial coefficient is 0 in case the second argument is non-integral.
% It is well known that the largest binomial coefficient is $\Theta(2^t/\sqrt{t})$, hence $\Prob{T_1 = t} = \Theta(t^{-3/2})$.
%, cf.\ the proof of Lemma~8 in~\cite{DoerrWinzen2014}.
It is known that $\Prob{T_1 \ge t} = \Theta(t^{-1/2})$. 
This can be derived as follows. By~\cite[III.7, Theorem~2]{Feller1968}
\[
\Prob{T_1 = t'} = \frac{1}{t'} \binom{t'}{\frac{t' + 1}{2}} \cdot 2^{-t'}
\]
where the binomial coefficient is 0 in case the second argument is non-integral.
For odd~$t'$ the above is at least $\Omega(t'^{-3/2})$.
Integrating over all odd values of $t' \ge t$ yields $\Prob{T_1 \ge t} = \Theta(t^{-1/2})$.

Let $c$ be the implicit constant in the upper bound of the $\Theta$ expression. For $\eta > 1$, in order for $T_\eta \ge t$, all $\eta$ random walks must not have reached the target in the first $t-1$ states. Since all random walks are independent, $\Prob{T_\eta \ge t} \le (\Prob{T_1 \ge t})^{\eta} \le c^{\eta} t^{-\eta/2}$.

For $\eta \ge 3$ it suffices to consider $\eta=3$ as $T_3$ stochastically dominates $T_\eta$ for $\eta \ge 3$. The expectation can then be derived as 
\[
\E{T_3} = \sum_{t} \Prob{T_3 \ge t} \le \sum_t c^{3} t^{-3/2}
= c^{3} \cdot O(1) = O(1).
\]
For $\eta=2$, we use the second statement of Lemma~\ref{lem:random-walk-tools} to infer that for all $r \in \mathbb{N}
$ and all $t \in [2rk^2, 2(r+1)k^2)$, we have $\Prob{T_2 \ge t} \le 2^{-r}$, thus $\sum_{t=2rk^2}^{2(r+1)k^2-1} \Prob{T_2 \ge t} \le 2k^2 \cdot 2^{-r}$. Thus, we get 
\begin{align*}
\sum_{t = 2rk^2}^\infty \Prob{T_2 \ge t} \le 2k^2 \sum_{s=r}^\infty 2^{-s} = 4k^2 \cdot 2^{-r}.
\end{align*}
Choosing $r := 2\log k$, this is at most $4$ and we get
\begin{align*}
& \sum_t \Prob{T_2 \ge t} = \sum_{t=1}^{(4k^2 \log k)-1} \Prob{T_2 \ge t} + \sum_{t=4k^2 \log k}^\infty \Prob{T_2 \ge t} \\
\le\;& \sum_{t=1}^{(4k^2 \log k)-1} \Prob{T_2 \ge t} + 4
\le \sum_{t=1}^{(4k^2 \log k)-1} c^2 t^{-1} + 4\\
=\;& c^2 H(4k^2 \log k) + 4 = O(\log k).
\end{align*}
A lower bound of $\Omega(\log k)$ follows from the fact that at least $k-1$ steps are needed to reach the reflecting state, and until then the process behaves as on an unbounded state space. Then $\E{T_2} \ge \sum_{t=1}^{k-1} \Prob{T_2 \ge t} \ge  \sum_{t=1}^{k-1} c' \cdot t^{-1} \ge c' H(k-1) = \Omega(\log k)$ where $c' > 0$ is the implicit constant in the lower bound of $\Theta(t^{-1/2})$.
\end{proof}

Now we are prepared to prove Theorem~\ref{thm:island-model}.
\begin{proof}[Proof of Theorem~\ref{thm:island-model}]
We show that the expected number of generations for finding a proper coloring after each edge insertion is $O(1)$ if a random walk is necessary. If an added edge leads to a conflict, the $\lambda$ islands perform $\lambda$ independent random walks as described in Lemma~\ref{lem:independent-random-walks}. Applying said lemma with $\eta := \lambda$ yields the claimed bound of $O(1)$ generations. Multiplying by $\lambda$ for the number of evaluations and summing over $m$ edge insertions yields the claim.
\end{proof}

\section{Conclusions}
Evolutionary algorithms have been applied to a wide range of dynamic optimization problems. We have shown that dynamic evolutionary optimization approaches can also be useful to solve a given static problem if the problem instance is fed to the algorithm in an incremental fashion. 

For 2-coloring bipartite graphs, the simple RLS is effective on all graph instances if the order of the edges is given based on popular graph traversals.
This includes graphs where RLS fails with an overwhelming probability in the static case. The order of edges provided is essential: for a worst-case order or a random order, RLS fails on trees with an overwhelming probability. However, every graph traversal leads to polynomial expected times. Comparing popular graph traversals like depth-first search and breadth-first-search shows that the latter is more effective as performance guarantees only depend on the diameter of the graph, whereas for the former they depend on the length of the longest simple path.

Furthermore, we have shown that offspring populations in the \lRLS lead to an exponential speedup for appropriate choices of~$\lambda$, since the probability of making the right decision for resolving a new conflict immediately is amplified.
Surprisingly, island models using parallel evolution to rediscover proper colorings are even more effective. With only 3 islands, the island model achieves the best possible runtime of $\Theta(m)$ for all graphs with $m$ edges. This is the first example of a proven speedup with islands that is not bounded in the number of islands. Island models are also more robust with respect to the choice of graph traversal and the graph instance as the expected time for the island model only depends on the number of edges, for every graph traversal and every graph. 

Future work could consider whether the incremental approach would also work on graphs with a larger number of colors and whether it proves useful for other combinatorial problems.

\section*{Acknowledgment}
This research has been supported by the Australian Research Council (ARC) through grant DP160102401.

\bibliographystyle{unsrt}  
\bibliography{coloring,reoptimization}  %%% Remove comment to use the external .bib file (using bibtex).

\end{document}